%% file: uai2021-template.tex
\documentclass[accepted]{uai2021} 

\usepackage[american]{babel}

\usepackage{natbib} 
\usepackage{mathtools} 
\usepackage{booktabs} 

\usepackage{balance} 
\usepackage{graphicx}
\usepackage{times}

\usepackage{soul}
\usepackage[utf8]{inputenc}
\usepackage{amsmath,amsthm,thmtools}
\usepackage{tikz}
\usetikzlibrary{patterns,arrows}
\usepackage{subfigure}
\usepackage{caption}
\usepackage{enumitem}
\usepackage{algorithm,algpseudocode}
\usepackage{amssymb}

\newcommand{\cM}{\mathcal{M}}
\newcommand{\cR}{\mathcal{R}}

\newcommand{\cD}{\mathcal{D}}

\newcommand{\cS}{\mathcal{S}}

\newcommand{\TripCost}{\mathsf{TripCost}}

\newcommand{\route}{\omega}

\newcommand{\TripLP}{\mathsf{TripLP}}
\newcommand{\TripSearch}{\mathsf{TripSearch}}

\newcommand{\nxt}{\mathtt{nxt}}
\newcommand{\CurPas}{\mathtt{Pas}}
\newcommand{\dist}{\mathtt{time}}

\newcommand{\stableSet}{\Xi}
\newcommand{\cN}{\mathcal{N}}
\newcommand{\val}{\nu}

\newcommand{\prob}{\varphi}
\newcommand{\PoF}{\mathsf{PoF}}
\newcommand{\SPoF}{\mathsf{SPoF}}
\newcommand{\PoS}{\mathsf{PoS}}
\newcommand{\RATIO}{\mathsf{R}}
\newcommand{\dschedule}{z}
\newcommand{\vardschedule}{a}

\newtheorem{proposition}{Proposition}
\newtheorem{observation}{Observation}
\newtheorem{definition}{Definition}
\newtheorem{claim}{Claim}

\title{Efficiency, Fairness, and Stability in Non-Commercial Peer-to-Peer Ridesharing}

\author[1]{Hoon Oh} 
\author[2]{\href{mailto: Yanhan (Savannah) Tang <yanhanta@andrew.cmu.edu>}{Yanhan (Savannah) Tang}{}}
\author[1]{Zong Zhang}
\author[3]{Alexandre Jacquillat}
\author[1]{\href{mailto: Fei Fang <feifang@cmu.edu>}{Fei Fang}}

\affil[1]{
    School of Computer Science\\
    Carnegie Mellon University\\
    Pittsburgh, Pennsylvania, USA
}
\affil[2]{
    Tepper School of Business\\
    Carnegie Mellon University\\
    Pittsburgh, Pennsylvania, USA
}
\affil[3]{
    Sloan School of Management\\
    Massachusetts Institute of Technology\\
    Cambridge, Massachusetts, USA
}

\begin{document}

\maketitle 
\begin{abstract}
Unlike commercial ridesharing, non-commercial peer-to-peer (P2P) ridesharing has been subject to limited research---although it can promote viable solutions in non-urban communities. This paper focuses on the core problem in P2P ridesharing: the matching of riders and drivers. We elevate users' preferences as a first-order concern and introduce novel notions of fairness and stability in P2P ridesharing. We propose algorithms for efficient matching while considering user-centric factors, including users' preferred departure time, fairness, and stability. Results suggest that fair and stable solutions can be obtained in reasonable computational times and can improve baseline outcomes based on system-wide efficiency exclusively.
\end{abstract}

\section{Introduction}
On-demand ride-hailing platforms have become increasingly popular in urban areas. However, the availability and affordability of on-demand transportation remain much more limited in suburban and rural areas. In practice, an overwhelming majority of commuting trips rely on self-driving with private vehicles and slow, infrequent public transportation. An increasingly popular option to promote alternative forms of mobility in non-urban areas lies in peer-to-peer (P2P) ridesharing: by bringing together commuters traveling along similar routes at similar times, P2P ridesharing can enhance mobility while reducing the costs of transportation, traffic congestion, and greenhouse gas emissions. Moreover, P2P ridesharing can also improve access to basic needs for disadvantaged populations with limited car ownership.

To be successful, P2P ridesharing platforms require effective algorithms to match rider requests with drivers. Matching in ridesharing platforms has attracted considerable research interest in recent years \cite{alonso2017demand, ozkan2020dynamic,bertsimas2019online,santi2014quantifying,bei2018algorithms,furuhata2013ridesharing,shah2019neural}, building upon related problems such as the dial-a-ride problem (DARP) \cite{cordeau2006branch,parragh2010variable} and the vehicle routing problem with time windows (VRPTW) \cite{cordeau2007vehicle}, matching in spatial-temporal networks, and is sometimes studied jointly with the topic of pricing \cite{bimpikis2016spatial}.
However, there is only limited research for P2P ridesharing without direct payments from riders to drivers~\cite{masoud2017real}, which can provide viable solutions for a community, e.g., residents from close-by regions and employees of the same company, featuring drivers who have their own travel plans and are willing to share part of their trips with riders.
Moreover, all existing work in matching in P2P ridesharing only considers the flexibility windows of the drivers and riders (henceforth, ``users'') as constraints and ignores users' preferences and incentives for participation. In addition, the predominant objective used in this setting is to minimize total costs (e.g., travel costs and inconvenience costs). However, such approaches do not capture the impact of matching decisions on individual users, including the fairness among users and whether or not the users will accept the matching outcome or continue participating on the platform.

In this paper, we address these limitations and study matching in P2P ridesharing without payment from a user-centric perspective, with the objective to balance system-wide efficiency and user satisfaction. To our knowledge, we are the first to study the efficiency-fairness-stability tradeoff in P2P ridesharing.
We make the following contributions. 1) We propose a new algorithm that combines tree search with linear programming to find optimal driver-schedule for individual drivers and enhances request-trip-vehicle (RTV) framework~\cite{alonso2017demand} to find the most efficient matching. 2) We formalize the notions of fairness and stability in P2P ridesharing and prove that the price of fairness ($\PoF$) and stability ($\PoS$) can be arbitrarily large. 3) We design algorithms for computing efficient solutions given fairness and stability constraints and evaluate them through extensive experiments.

We have implemented the system as a web service and delivered it to a rural US county with 91,000 residents, many of whom work in an urban area that is 1-hour drive away. The platform supports a non-commercial and not-for-profit P2P ridesharing program that is supported by the government to ease commutes from/to work and other essential needs. We are still waiting to get post-deployment statistics from them. Also, we are working with a community of 100+ families living in subsidized, low-rent apartments in a suburban area in the US for future deployment. Twenty residents have already expressed their interest in participation in a recent survey.

\subsection{Additional Related Work}
Well-studied notions of fairness include max-min fairness~\cite{bertsimas2011price} and envy-freeness~\cite{bogomolnaia2001new}.
In a probabilistic setting, \cite{liu2003opportunistic} limits the probability that the difference in individual users' utility is too large. Recent work on randomized online matching demonstrates fairness by analyzing the marginal probability of a pair of items being matched~\cite{RandomizedOnline2018}. In this paper, we propose a novel notion of fairness that combines these ideas and is based on the max-min marginal probability of a user being matched.

Stability is well studied in two-sided matchings \cite{manlove2002hard,iwama2008survey}. including ridesharing with payment~\cite{bistaffa2015recommending}. However, stable matching in ridesharing without payments has not been studied until recently \cite{wang2017stable} and mainly in a simple setting with unit vehicle capacity ---which contrasts with our multi-capacity setting.

Finally, our multi-objective framework relates to the notions of the price of fairness and the price of stability in resource allocation \cite{bertsimas2011price,anshelevich2008price}. To our knowledge, the joint relationships between efficiency, fairness, and stability have not been studied in the context of P2P ridesharing.

\section{Efficient Matching with User Preferences}\label{sec:effmtch}

We first define the P2P ridesharing problem, and then present our algorithm for finding an efficient matching with user preferences. A notation table is available in the appendix attached to this paper.

\subsection{Model}

Let $\cR$ and $\cD$ be the set of riders and drivers respectively. The set of users is $\cR\cup\cD$. A user $i$ is characterized by his origin $o_i$, destination $q_i$, a time window $W_i = [\tau^e_i,\tau^l_i]$ describing the earliest possible departure time and latest possible arrival time, and value of the trip $\val_i$. In addition, we consider the user's preferred departure time $\tau^\star_i \in W_i$ and his maximum acceptable detour time $\Delta_i$.
Moreover, we denote by $k_d$ the capacity of driver $d$ (i.e., number of seats for riders in his vehicle). 

We consider a finite, continuous time
horizon $[0,T]$. We denote the set of users' origins and destinations by $V:= \{o_i \cup q_i : i\in \cR \cup \cD\}$.  
Let $\dist(u,v)$ be the shortest travel time from $u\in V$ to $v\in V$. It should be noted that if any two users' origins and/or destinations are co-located, we treat them as two different elements in $V$ with $\dist(\cdot)=0$. Thus each element in $V$ is associated with a user and is specified whether it is an origin or a destination. We refer to the elements as locations for simplicity. Let $\dist_i := \dist(o_i,q_i)$ be the user $i$'s default travel time.


If a rider $r$ is not matched, he can complete the trip with a cost $\lambda_r$, which can be seen as the cost of an alternative transportation mode.
For matched riders and all drivers, a user $i$ incurs a cost $C^i_{tt'}$ when he leaves his origin at time $t$ and arrives at $t'$. Following~\cite{alonso2017demand,wang2017stable}, we assume
\begin{align}
    C^i_{tt'}&:= c^i_{\text{dev}}\cdot|t-\tau^\star_i|+c^i_{\text{trl}}\cdot(t'-t)
    \label{eqn:cost}
\end{align}
where $c^i_{\text{dev}}$ is the cost per unit of deviation from preferred departure time and $c^i_{\text{trl}}$ is cost per unit of traveling time. 
We assume $\lambda_r\in[ C^r_{\tau^\star_r, \tau^\star_r+\dist_r},\val_r]$.

A \emph{driver-schedule} $\dschedule=\langle (v,t)|v\in V,t\in[0,T]\rangle$ is an ordered sequence of location-time pairs describing how a single driver travels to pick up and drop off riders.
A driver-schedule is feasible if it can be implemented sequentially and it satisfies all the constraints of all users involved. 
A \emph{stop} is a node in a driver-schedule. 
The driver and the subset of riders associated with a driver-schedule can be easily identified from the stops due to the uniqueness of nodes in $V$. 
Let $\cS=2^\cR$ be the set of all subsets of riders.
We say a $(d,S)$ pair, where $d\in \cD$ and $S\in\cS$, is \emph{feasible} if there exists a feasible driver-schedule for driver $d$ to pick up and drop off all riders in $S$.
A \emph{system-schedule} or \emph{schedule} for short, is a collection of driver-schedules, one for each driver, with each rider shown up in at most one driver-schedule. Let $\Pi$ be the set of all possible schedules.
Given a schedule $\pi\in\Pi$, the subset of riders that driver $d$ is matched to is denoted by $S^\pi_d$.

A \emph{matching} is an assignment of drivers $\cD$ to subsets of riders $\cS$ such that each driver is assigned to exactly one subset, and each rider is assigned to at most one driver. Let $\cM$ be set of all possible matchings. 
Each schedule $\pi$ defines exactly one matching, and a matching $M\in \cM$ may correspond to multiple schedules. Thus, we sometimes use $\pi$ to refer to both a matching and a schedule. Denote by $S^M_d$ the subset of riders that driver $d$ is matched to.

If our goal is to maximize system-wise efficiency, the objective of the matching problem is to find an optimal schedule $\pi^*$ that minimizes the total cost in the system, i.e.,
\begin{align}\label{costObjectiveOri}
\pi^*=\arg\min_{\pi \in \Pi} \sum_{i\in \cD\cup_{d\in \cD} S^\pi_d} C^{i,\pi} + \sum_{r\in \cR\setminus \cup_{d\in \cD} S^\pi_d} \lambda_r
\end{align}
where $C^{i,\pi}$ is the total cost for user $i$ given schedule $\pi$, computed by first extracting the pickup and dropoff time of user $i$ and then following the cost definition in Eqn \ref{eqn:cost}.

\subsection{Algorithm for Maximizing Efficiency}

 Let $c_{dS}$ be the minimum cost of a driver $d$ and a set of riders $S$ if $d$ is asked to serve all riders in $S$. $c_{dS} = \infty$ if $(d,S)$ is infeasible. 
Then the objective in (\ref{costObjectiveOri}) can be rewritten as \begin{align}\label{costObjective}
{\textstyle \min_{\pi\in\Pi} \sum_{d\in \cD}c_{dS^\pi_d} + \sum_{r\in \cR\setminus \cup_{d\in \cD} S^\pi_d} \lambda_r}
\end{align}
Thus the problem of finding the optimal system schedule can be decomposed into two subproblems: 

\noindent\textbf{(SP1)} Given a $(d,S)$ pair, compute $c_{dS}$ and find the optimal driver-schedule $\dschedule_{dS}$ for $d$ and return $\infty$ if infeasible;

\noindent\textbf{(SP2)} Using SP1 as a subroutine, find the optimal matching and system schedule. SP1 is NP-hard through a reduction from the \textit{Traveling Salesman Problem (TSP)} (proofs are in appendix).

The state-of-art RTV framework~\cite{alonso2017demand} seems suitable for the problem but directly applying it will suffer from two major limitations.
First, RTV solves SP1 through exhaustive search or heuristic methods, which are time-consuming or suboptimal, and are not directly applicable to our problem due to the key element of user' preferred time in a continuous-time horizon. This limitation is magnified considering that SP1 will be called many times in SP2.
Second, in solving SP2, RTV calls SP1 solver to check feasibility and compute cost for all the feasible $(d,S)$ pairs in a neat order to reduce the runtime, but it fails to leverage the similarity across $(d,S)$ pairs to reduce the runtime for solving SP1 for $(d,S)$ pairs that are checked later in the process. We develop $\TripCost$ algorithm to address the first limitation and propose two enhancements to the RTV framework to mitigate the second limitation.

$\TripCost$ (Alg \ref{alg:tripcost}) is a depth-first search-based algorithm that solves a linear program $\TripLP$ at each leaf node. It first finds a heuristic driver-schedule $\dschedule_h$ and its cost $c_h$ (Line \ref{tripcost:initialization}). To do so, it shuffles the pickup and drop-off order of the riders to get a driver-schedule without specified stopping times, i.e. a route $\route_{\text{random}}$, and solves a linear program $\TripLP$, which finds the best time to visit each stop. 
Then it calls $\TripSearch$ to build a search tree.
Each node of the search tree corresponds to a stop in a driver-schedule, with the root representing $o_d$, i.e. the origin of the driver $d$, and the leaf node being $q_d$, i.e. the destination of $d$. The path from the root to a node represents a partial route.
When reaching a leaf node during the tree search, we get a complete route and call $\TripLP$ again to determine when to visit each stop. 
For intermediate nodes, we expand the node by appending a feasible unvisited stop to the current partial route.

$\TripLP$ is built upon the following observation.
\begin{observation}\label{OBS1:EQUIVALENCE}
There exists an optimal driver-schedule where the driver only waits at the pickup location of a rider (to satisfy the rider's time window and adapt to the rider's preferred departure time) and always takes the shortest path to reach the next stop. 
\end{observation}
Therefore, it is sufficient to determine the departure time at each stop. We use variables $a_v$ to represent the departure time at stop $v\in \{o_i,q_i|i\in d\cup S\}$ and get $\TripLP$: 
\begin{align}
\min_{\vardschedule} \quad& \sum_{i\in d\cup S} c^i_{\text{dev}}|\vardschedule_{o_i}-\tau^*_i|+c^i_{\text{trl}}(\vardschedule_{q_i}-\vardschedule_{o_i})  \label{tripLPobj}\\
\text{s.t.}\quad
& \vardschedule_{v}+\dist(v,\nxt(v)) \leq \vardschedule_{\nxt(v)}, \forall v \label{tripLP_timeSequence}\\
& \vardschedule_{o_i}\geq \tau^e_i, \forall i\label{tripLP_e}\\
& \vardschedule_{q_i}\leq \tau^l_i, \forall i\label{tripLP_l}\\
& \vardschedule_{q_i}-\vardschedule_{o_i}\leq \Delta_i, \forall i\label{tripLP_detour}
\end{align}
$\nxt(v)$ denote the stop after $v$ in the given route. The objective is to minimize the total cost for the driver and the riders (Eqn \ref{tripLPobj}). The absolute value term can be converted to linear constraints by following standard techniques. Eqn \ref{tripLP_timeSequence} ensures that each stop is visited and the time interval between visiting every two stops is no less than the traveling time. Moreover, each user has to be served within his feasible time window (Eqn \ref{tripLP_e} - \ref{tripLP_l}) and maximum detour time (Eqn \ref{tripLP_detour}). $\TripLP$ has $2|S|+2$ variables and $5|S|+4$ constraints.
\begin{algorithm}[t]
  \caption{TripCost($d,S)$}\label{alg:tripcost}
  \begin{algorithmic}[1]
  \State $(c_h,z_h)\gets \mathsf{Solve}\TripLP(\route_{\text{random}}(S))$ \label{tripcost:initialization}
  \State \Return $\mathsf{TripSearch} (d,S,c_h,\dschedule_h,\langle o_d\rangle,\tau_d^e,0)$ \label{tripcost:treesearch}
  \end{algorithmic}
\end{algorithm}

$\TripSearch$ uses several pruning techniques, including using $c_h$ as an upper bound of $c_{dS}$ and using capacity and time constraints to cut unpromising branches. 

Given $\TripCost$ as a SP1 solver, we can now find the efficiency-maximizing system schedule. We will build upon the RTV framework and use two new enhancements: driver schedule reusing and user decomposition.
We provide an overview of the RTV framework for completeness, and then introduce our enhancements.
RTV first enumerates all the feasible $(d,S)$ pairs in an incremental manner. It builds upon the observation that $(d,S)$ is feasible only if $(d,S')$ is feasible for all $S'\subseteq S$. So for each driver $d$, we construct all the sets of riders that are compatible with $d$ by gradually increasing the size of the set. In each step, for each existing feasible set $S$ of size $h-1$, we add one rider to get a set $S'$ of size $h$. Only if all the size-$(h-1)$ subsets of $S'$ already exist in the list, $\TripCost$ is called to further verify the feasibility and compute the cost. 
Then RTV uses the following binary integer program (BIP) to find the optimal matching.
\begin{align}
    \min_{x,y} \quad & \sum_{d\in\cD}\sum_{S\in\cS} c_{dS}x_{dS}+\sum_{r\in \cR} \lambda_r y_r\label{MatMILP}\\
    s.t. \quad & \sum_{d\in \cD}\sum_{S\in \cS| r\in S}x_{dS}+y_r=1, \forall r \label{MatMILPpasCon}\\
    \quad & \sum_{S\in \cS} x_{dS} = 1, \forall d \label{MatMILPdriCon}\\
    \quad & x_{dS},y_r\in\{0,1\}, \forall d, S, r\label{MatMILPintCon}
\end{align}
$x_{dS} = 1$ iff $(d,S)$ is matched; and $y_r=1$ iff rider $r$ is not matched to any driver. 
More details about the RTV framework can be found in the appendix.

\textbf{Enhancement 1: Driver-Schedule Reusing} In the RTV framework, when we try to compute $c_{dS}$ for a $(d,S)$ pair, we have already found the best driver-schedules for (i) $(d,S')$ where $S'\subset S$ and $|S'|=|S|-1$; and sometimes also (ii) $(d',S)$ where $d'\neq d$. Thus, we can \emph{reuse}, or \emph{learn from} those driver-schedules to find a promising heuristic driver-schedule for $(d,S)$, which can be used for pruning when calling $\TripCost$ (i.e., provide a better $(c_h,z_h)$ in Alg \ref{alg:tripcost}). This can be viewed as a warm start for $\TripCost$. For (i), we can insert the new rider into the best route for $S'$ and solve the $\TripLP$ to find a new driver-schedule. 
Further, for two subsets of $S$, $S_1$ and $S_2$ such that $|S_1| = |S_2| = k-1$ and $S_1 \cup S_2 = S$, if riders in $S_1 \cap S_2$ share similar route ordering, we closely follow the common ordering to get a promising route.
For (ii), we can directly reuse the order of visiting the stops and recompute the timing if the $L_\infty$-norm of origin and destination of $d$ and $d'$ is within a small threshold $\varepsilon$. 




\textbf{Enhancement 2: User Decomposition}
When the size of the problem is huge, we propose to first decompose a problem instance into several mutually independent sub-instances and then solve those smaller sub-instances in parallel. To construct such sub-instances, we start by treating each driver as its own group. Then we go through all the riders. If a rider is compatible with two drivers $d$ and $d'$ and the two drivers are not in the same group, then we merge the two groups $d$ and $d'$ are in.
We continue the process until we checked all the riders. The process decomposes the driver set into groups such that no two drivers from different groups can be compatible with the same rider. The decomposition further helps improve the scalability of our algorithm, especially when the system is expanded over different geographic regions. We provide more details of the algorithms in appendix.

\section{Modeling Fairness and Stability}
We now introduce notions of fairness and stability to make our matching outcomes consistent with user preferences. 

\subsection{Utility model}
Each rider's utility is defined as (\textit{the value of the trip}) $-$ (\textit{the cost incurred}). If rider $r$ is picked up at $t$ and dropped off at $t'$, his/her utility is $U_r = \val_r - C^r_{tt'}$. If the rider is not matched, then $U_r = \val_r - \lambda_r$. Define a driver's \emph{base utility} as $\bar{U}_d=\nu_d-C_{tt'}^d$. Giving rides to others only leads to a lower utility for the driver than driving alone under this definition. Then why would he participate in rideshairng? In other non-monetary P2P systems~\cite{bellotti2015muddle,vassilakis2009analysis}, altruism is considered to be an important motivation. Thus, altruism may account for drivers's participation in our setting as well, especially when drivers are helping their own community. In addition, these non-commercial ridesharing platforms are often run by the local government or big employers, who may provide external rewards,such as coupons from in-community stores, parking fee discounts, to encourage participation.
Therefore, we assume that the driver gains an extra utility proportional to the utility gained by their matched riders, either due to altruism or external rewards. 
Let $U_d := \bar{U}_d + \rho_d \sum_{r\in S_d} U_r$ be a driver's  \textit{actual utility} where $\rho_d$ is the extra utility factor. Let $U^i_{dS}$ be the (actual) utility that the user $i$ gets when he is part of the $(d,S)$ pair, i.e. $i\in S$ or $i = d$.
In this paper, we use total cost as our measure of efficiency due to the following observation. All results and algorithms can be easily extended for social welfare maximization.
\begin{observation}{}
 The cost minimization problem is equivalent to maximizing the sum of base utilities.
 \label{obs1:equivalence}
\end{observation}
Although, in general, a cost-minimizing solution may not be a welfare-maximizing solution, they are the same in quasi-linear utility models. 
The full proof of Observation~\ref{obs1:equivalence} is deferred to the appendix. 

\subsection{Fairness Model}
Let $\{M^\ell\}_1^{\eta}$ be the set of all feasible deterministic matching. Let $m^\ell_i\in\{0,1\}$ indicate whether that rider $r$ is matched in a deterministic matching $M^l$. Let $p^\ell$ be the probability of choosing a deterministic matching $M^\ell$.   
Then $\prob_r(\cM):= \sum_{\ell\in[\eta]} p^\ell m^\ell_r$ is the probability that a rider $r$ is matched in a probabilistic matching $\cM=\langle p^1,p^2,...,p^\eta\rangle$. Let 
$Cost(M^\ell)$ be the total cost of a deterministic matching $M^\ell$, i.e. $Cost(M^\ell) = \sum_{(d,S) \in M^\ell} c_{dS} + \sum_{r} \lambda_r(1-m^\ell_r)$. Then, the expected cost of a probabilistic matching $\cM$ is $Cost(\cM)=\sum_{l\in[\eta]} Cost(M^\ell)p^\ell$.

We formalize fairness by maximizing the lowest probability of matching, across all riders in the system. In other words, we want to maximize $\min_{r\in \cR} \prob_r$. 
Note that if there exists a rider $r$ that cannot be feasibly matched, then the value is always 0. Thus, we only focus on riders that can be matched to a driver. 
With a slight abuse of notation, $\cR$ is henceforth the set of feasible riders.
Our notion is very similar to the well-known fairness metric -- max-min fairness. However, unlike common max-min fairness notion, we consider all riders to be equal (independent of its utility). In P2P platforms, we do not want to discriminate any rider; therefore, we define fairness only based on the probability of getting matched.

With this motivation, we now construct a probabilistic matching that ensure matching for all feasible riders.  Many P2P system is long-term. Therefore, it is reasonable to consider probabilistic matching.
Our definition can be generalized by having different $\prob_r$ threshold for different riders based on their flexibility. This different $\prob_r$ can give non-monetary incentive for different riders; however, this is not the focus of this paper.


\begin{definition}
A probabilistic matching $\cM$ is \emph{$\theta$-fair} if $\prob_r(\cM)\geq \theta$ for all riders $r\in \cR$. 
\end{definition}

\subsection{Stability Model}
We now turn to stability. 
Let us first define stability at the individual level, which is essentially individual rationality.
\begin{definition}
A matching satisfies \emph{individually rationality} (IR) for a user if he does not get a worse utility by participating in the P2P system. 
\end{definition}
In other words, a matching satisfies IR for rider $r$ if $U_r \geq \val_r - \lambda_r$, and IR for driver $d$ if $U_d \geq \val_d - c^d_{\text{trl}}\cdot\dist_d$. Note that the stability of a matching for a driver relies on extra utility. 

 We now extend the idea of IR to define stability at the group level.
Specifically, we ensure that no group of users can benefit from forming an alternative matching outside of the P2P platform. 
 \begin{definition}
 A $(d,S)$ is a \emph{blocking pair} of a matching $M$ if $(d,S)$ is currently not matched in $M$ but 
 $U^d_{dS} > U^d_{d,S_d^M}$
 and $U^r_{dS} > U^r_{d_r^M,S^M_r}$
 for all $r\in S$, where $S_d^M$ is the set of riders that are matched to $d$ under $M$, $d_r^M$ is the driver that rider $r$ is matched to in $M$, and $S^M_r$ is a subset of riders that are matched to the same driver with $r$ (including $r$) in $M$.
 \end{definition}
 \begin{definition}
 A matching $M$ is \emph{stable} if it has no blocking pair.
 \end{definition}



 \begin{definition}
 A probabilistic matching is \emph{ex-post stable} if each matching assigned a non-zero probability is stable. 
 
 \end{definition}

\section{Efficiency-Fairness-Stability Trade-offs}

We now extend the algorithm from section \ref{sec:effmtch}, to incorporate fairness and stability. We also provide theoretical results on the trade-offs between efficiency, fairness, and stability.

\subsection{Efficiency-Fairness Trade-off}


We provide a LP (Eqn \ref{fairnessLP} - Eqn \ref{probabilityConstraint2}) to compute a $\theta$-fair probabilistic matching which has the minimum total cost among all $\theta$-fair matching. 
\begin{align}
    \min_p \quad & \textstyle\sum_{\ell\in [\eta]} Cost(M^\ell)p^\ell \label{fairnessLP}\\
    \text{s.t.}\quad& \textstyle\sum_{\ell\in[\eta]} m_i^\ell p^\ell \geq \theta & \forall i\in \cR\label{thetaConstraint}\\
    & \textstyle\sum_{\ell\in [\eta]} p^\ell = 1 \label{probabilityConstraint}\\
    & p^\ell \geq 0 &\forall \ell\in [\eta] \label{probabilityConstraint2}
\end{align}

As the problem scale increases, it can be challenging to even just enumerate all the variables as the number of possible matchings, $\eta$, can be exponentially large. 
Instead of directly solving the LP, we resolve the scalability issues by following the column generation method and incrementally add matchings one by one.
In each iteration, we solve the master problem which is primal LP \eqref{fairnessLP} - \eqref{probabilityConstraint2} with a subset of matchings and obtain dual variables $w$ of constraint \eqref{thetaConstraint} and $\alpha$ of constraint \eqref{probabilityConstraint}. Then we solve the slave problem to find a matching $M^{\ell^*}$ to be added to the master problem, which maximizes the dual objective $\sum_i m_i^\ell w_i + \alpha - Cost(M^\ell)$. Since $\alpha$ does not change with the matching, we have
\begin{align*}
\ell^*=\arg\min_{\ell\in[\eta]} Cost(M^\ell) - \sum\nolimits_{i\in M^\ell} w_i
 \end{align*}
Without the second term, this optimization problem is just the one solve in \eqref{MatMILP} - \eqref{MatMILPintCon}. With the second term, the matching can be found by a similar BIP, and the only difference that $c_{dS}$ in the objective function is replaced by $c_{dS}-\sum_{i \in S\cup \{d\}}w_i$. 

One missing piece is to find an initial feasible matching to bootstrap the column generation. Because $\exists \theta^*\in [0,1]$, such that \eqref{fairnessLP}-\eqref{probabilityConstraint2} is feasible iff $\theta\leq\theta^*$ --- we want to find $\theta^*$, the largest possible fairness level. This task is nontrivial, we again are facing an exponential number of feasible matchings. Thus, we solve the following LP \eqref{maxThetaLP}, and again through column generation. Note that $\theta=0$ is a feasible solution for the LP \eqref{maxThetaLP}, thus there is a trivial feasible solution to bootstrap the column generation for it. Denote the optimal solution for the LP \eqref{maxThetaLP} by $p^0$ and $\theta^0$. Then $\theta^0$ is the maximum level of fairness that any probabilistic matching can achieve. Therefore $p^0$ provides a feasible matching to the LP (\ref{fairnessLP}) or LP (\ref{fairnessLP}) is infeasible.
\begin{align}
    \max\nolimits_{p,\theta} \quad  \theta
     \qquad
     \text{s.t.}\quad \eqref{thetaConstraint}-\eqref{probabilityConstraint2}    \label{maxThetaLP}
\end{align}








For a feasible probabilistic matching $\cM$, it is associated with a cost $Cost(\cM)$ and a fairness level $\theta(\cM) = \min_r \prob_r(\cM)$. The trade-off between efficiency and fairness can be easily illustrated by the Pareto frontier, which consists of all the $(\theta, Cost)$ pairs. The Pareto frontier characterizes the trade-off between fairness ($\theta$) and efficiency ($Cost$) in the P2P ridesharing problem.
\begin{proposition}
The Pareto frontier of the P2P ridesharing problem a) is piece-wise linear; b) is convex in $\theta$; and c) can be computed in polynomial time with respect to the number of feasible matchings and optimal trips. 
\label{propPoFlinearity}
\end{proposition}
\begin{proof}
The proof of (a) and (b) follows from global sensitivity analysis from \cite{bertsimas1997introduction}. A full proof is differed to the appendix.

For (c), we consider a bisection search algorithm that finds the exact Pareto Frontier. This method applies to a broader class of problems where the Pareto Frontier is known to be a piece-wise function of finite sub-functions and has non-decreasing second-order gradients. The bisection search algorithm first finds the envelope hyper-planes of the Pareto frontier at boundary points $\theta^0_0 = 0$, and $\theta^0_1 = 1$. 
Without loss of generality, we assume that the two lines intersect, we record the $\theta$-value($\theta^1_0$) and cost of the intersection of the two lines.  We compute the optimal cost at $\theta^1_0$, and compare the optimal cost with the cost of the intersection. From (b) we know that the first-order difference of the Pareto frontier is non-decreasing, thus the true optimal cost at $\theta^1_0$ could either be equal to or greater than the cost value of the intersection point. If equal, then the upper border of the two intersecting lines is the Pareto frontier on the interval; otherwise, we bisect the interval, and repeat the procedure described above at both of the two half-length closed intervals. Note that there are only finite number of line intersections and base changes in the linear optimization program, as a result, this algorithm will stop within finite number of bisection search. 
\end{proof}



We quantify the price of fairness and strong price of fairness, denoted by $\PoF(\theta)$ and $\SPoF(\theta)$, respectively. $\PoF(\theta)$ and $\SPoF(\theta)$ are defined as the best- and worst-case increase in the system's cost when fairness considerations are included, respectively. Let $\mathsf{M}_F(\theta)$ be the set of all feasible $\theta$-fair probabilistic matchings\footnote{It is possible that for some large $\theta$, there is no feasible fair matching that can achieve a fairness level of $\theta$. $\PoF$ and $\SPoF$ are only well-defined when $\mathrm{M}_F(\theta)\neq \emptyset$.}. 
\begin{definition}
$\PoF(\theta) = \min_{\cM^F\in \mathsf{M}_F(\theta)}\frac{Cost(\cM^F)}{Cost(\cM^*)}$, where 
$\cM^*$ is a cost-minimizing (probabilistic) matching.
\end{definition}

\begin{definition}
$\SPoF(\theta) = \max_{\cM_s^F\in\mathsf{M}_F(\theta)}\frac{Cost(\cM_s^F)}{Cost(\cM^*)}$, where $\cM^*$ is a cost-minimizing (probabilistic) matching.
\end{definition}

Let $\cS_d$ be the set of all possible subsets of riders that a driver $d$ can be matched to under IR and other feasibility constraints. To analyze the upper bounds of $\PoF(\theta)$, we define $\RATIO_d = \max_{S\in \cS_d}\{\frac{c_{dS}}{c_{d,S^*_d}}\}\geq 1$. Let $S_d\subseteq \cR$ be the subset of riders that can be possibly matched to $d$, i.e. $S_d=\{r\in\cR:r\in S, S\in\cS_d\}$. Let $S^*_d$ be the subset of riders matched to $d$ in min-cost matching.
\begin{proposition}\label{prop:PoF} 
$\PoF(\theta)\leq \max\{\max\limits_{d\in\cD: |S^*_d|>0}\{\theta(|S_d|-1)[\RATIO_d-1]+1\},\max\limits_{d\in\cD: |S^*_d|=0}\{\theta|S_d|(\RATIO_d-1)+1\}\}$ for $\theta\in[0,\frac{1}{|S_d|}]$. The upper bound is tight for $\theta\in[0,\frac{1}{|S_d|}]$. When $\theta = 0$, $\PoF(0)=1$.
\end{proposition}

\begin{proof}
The intuition comes from a simple case where $|\cD|=1$ and $|S^*_d| > 0$, serving any other subset of riders will incur a cost at most $\RATIO_d \cdot Cost(S^*_d)$. Thus, when $|\cD|=1$, $\PoF\leq \theta(|S_d|-1)\RATIO_d+[1-\theta(|S_d|-1)]=\theta(|S_d|-1)(\RATIO_d-1)+1$, if $S^*_d\neq\emptyset$;
on the other hand, $\PoF\leq\theta|S_d|\RATIO_d+[1-\theta|S_d|]=\theta|S_d|(\RATIO_d-1)+1$, if $S^*_d=\emptyset$. 
This is because the largest $\PoF(\theta)$ occurs in the instance described below, serving any other rider in the system will cost the driver $\RATIO_d Cost(S^*_d)$. 
For any driver, he can serve up to $|S_d|$ riders with probability $\theta\leq\frac{1}{|S_d|}$. The worst-case scenario is that the driver needs to serve all the other equally costly riders to assure fairness, while his best choice is not to serve anyone or serve the least costly one. When $|\cD|\geq 1$, the worst-case scenarios for $\PoF$ is when there are $|\cD|$ disjoint driver-rider subsystem in which each driver needs to serve all the customers in his rider subsets with certain probability to meet $\theta$-fairness. 

In a more general case, when $|\cD|\geq 1$ and 
$|S^*_d|>0$, let $\tilde{\theta} := \theta(|S_d|-1)$ be the probability that the driver is not assigned to the min-cost matching under the cost minimizing $\theta$-fair solution. Then $\PoF(\theta)\leq 
\tilde{\theta}(\RATIO_d) + (1-\tilde{\theta})=\tilde{\theta}(\RATIO_d-1)+1=\theta(|S_d|-1)(\RATIO_d-1)+1$.
When $|\cD|\geq1$ and 
$|S^*_d|=0$, $\PoF(\theta)\leq \theta|S_d|\RATIO_d+(1-\theta|S_d|)$ at level $\theta\in (0,\frac{1}{|S_d|}])$. Both upper bounds are tight as long as $\theta\in [0,\frac{1}{|S_d|}]$.
When $|\cD|>1$, the worst case scenario is when $\cap_{d\in\cD} S_d = \emptyset$. The upper bound of $\PoF(\theta)$ is the largest among all individual's $\PoF_d(\theta), d\in\cD$.
\end{proof}

\begin{proposition}\label{prop:SPoF}
In the P2P ridesharing problem, $\SPoF(\theta)\leq \max\{\max\limits_{d\in\cD:|S^*_d|>0}\{\frac{(|S_d|-1)(\RATIO_d-1)}{|S_d|}+1\}, \max\limits_{d\in\cD:|S^*_d|=0}\{\RATIO_d\}\}$, $\theta\in [0,\frac{1}{|S_d|}]$. The upper bound is tight at $\theta=\frac{1}{|S_d|}$. 
\end{proposition}

\begin{proof}
Based on the definition of $\SPoF(\theta)$ and a similar analysis of Proposition \ref{prop:PoF}, $\SPoF(\theta)\leq \max_{\theta'\geq \theta}\{\theta'(|S_d|-1)\RATIO_d+1-\theta'[(|S_d|-1)]\} = \frac{(|S_d|-1)(\RATIO_d-1)}{|S_d|}+1$ at level $\theta\in [0,\frac{1}{|S_d|}]$. If no rider is matched to $d$ in the min-cost matching, i.e. $S^*_d=\emptyset$, then the largest $\SPoF(\theta)$ occurs when serving any other rider in the system will cost the driver $\RATIO_d Cost(S^*_d)$. Thus, $\SPoF(\theta)\leq \max_{\theta'\geq \theta} \theta'[|S_d|\RATIO_d]+1-\theta'|S_d| = \max_{\theta'\geq \theta} \theta'|S_d|(\RATIO_d-1)+1=|\RATIO_d|$ at level $\theta\in (0,\frac{1}{|S_d|}]$. Both the upper bounds are tight as long as $\SPoF(\theta)$ is well defined.
If $|\cD|\geq1$, the worst case scenario is when $\cap_{d\in\cD} S_d = \emptyset$. Based on this observation, we have a tight upper bound 
$\SPoF(\theta)\leq \max\{\max_{d\in\cD, |S^*_d>0|}\{\frac{(|S_d|-1)(\RATIO_d-1)}{|S_d|}+1\},\max_{d\in\cD, |S^*_d=0|}\{\RATIO_d\}\}$ at level $\theta\in (0,\frac{1}{|S_d|}]$.
\end{proof}
Proposition \ref{prop:PoF} and \ref{prop:SPoF} establish the relationship between driver schedule cost ratio and system-wise $\PoF(\theta)$ and $\SPoF(\theta)$. When the cost ratio of the most costly feasible driver schedule(s) to the min-cost schedules ($\RATIO_d$) is bounded, $\PoF(\theta)$ and $\SPoF(\theta)$ are bounded. 

When constructing the RTV graph, we have stored $S_d$ and computed upper bounds of the costs of all rider subsets in $S_d$. Thus $S_d$ and $\RATIO_d$ are directly available. 
When the feasibility and IR constraints are very strong, we can alternatively estimate upper bounds for $\RATIO_d$ and obtain upper bounds for $\SPoF(\theta)$ and $\PoF(\theta)$. 

\begin{proposition}\label{prop:PoFlarge}
There exist problem instances where 
$\SPoF(\theta)\geq\PoF(\theta) = K$, $\forall K \geq 1$.
\end{proposition}

\begin{proof}
Proposition \ref{prop:PoFlarge} states that $\PoF$ may be arbitrarily large. Consider the following instance with two riders $r_1$, $r_2$ where $\lambda_{r_1} = \lambda_{r_2} = \varepsilon$, and one highly-altruistic driver $d$. The driver incurs a cost of $1-\varepsilon$ driving alone or serving $r_1$, a cost of $C-\varepsilon$ serving $r_2$, and a cost of $C+1$ serving both $r_1$ and $r_2$. The cost minimizing solution is to only serve $r_1$ with a cost of $1$, while the $\theta$-fair solution is to serve $r_1$ with probability $1-\theta$ and $r_2$ with probability $\theta$, which results in a cost of $C\theta +1-\theta$ for any $C\geq1$. When $\theta\in(0,1]$, $\SPoF\geq\PoF=C\theta +1-\theta=O(C)$.
\end{proof}
\subsection{Efficiency-Stability Trade-off}

It is known that a stable matching always exists in traditional ridesharing models \cite{wang2017stable}. However, because of drivers' extra utilities ($\rho_d,d\in\cD$) and the user-centric factors like preferred time ($\tau^{\star}_i,i\in\cD\cup\cR$), riders who are substitutes to each other otherwise, are actually complementary to each other in the P2P ridesharing problem. In other words, a rider may incur a lower additional cost when matched to the driver together with another rider. Therefore, the stable matching may not always exist in our model.

\begin{proposition}
In the P2P ridesharing problem, the set of stable outcomes may be empty.
\label{proposition:unstable}
\end{proposition}

Details of an example where a stable outcome does not exist are deferred to the appendix.




To compute a cost-minimizing matching that satisfies both IR and group-level stability, we only need minor changes to the efficiency-maximizing algorithm in section \ref{sec:effmtch}. For the IR constraint, we add constraints  $U^d_{dS}x_{dS} \geq \val_d - c^d_{\text{trl}}$ and $U^r_{dS}x_{dS} \geq \val_r - \lambda_r$ for all $(d,S)$ pairs in the matching BIP \eqref{MatMILP}-\eqref{MatMILPintCon}. For group-level stability, we define the following constraints for all feasible $(d,S)$ to the matching BIP \eqref{MatMILP}-\eqref{MatMILPintCon} to ensure that the matching is stable.
\begin{align}
    \sum_{\substack{S' : \\U^d_{dS'} > U^d_{dS}}} x_{dS'} + \sum_{r\in S}\sum_{\substack{d',S' :\\ U^r_{d'S'} > U^r_{dS}}} x_{d'S'}
    +\sum_{\substack{r\in S:\\ U^r_{\emptyset} > U^r_{dS}}} y_r+x_{dS} \geq 1
\end{align}

We also quantify how stability constraints impact efficiency. The price of stability ($\PoS$) is defined similarly to $\PoF(\theta)$, measuring the cost increase due to stability constraints.
 \begin{definition}
 $\PoS = \max_{\cM^S}\frac{Cost(\cM^S)}{Cost(\cM^*)}$, where $\cM^*$ are the cost-minimizing (probabilistic) matchings, and $\cM^S$ is a cost-minimizing stable (probabilistic) matching.
 \end{definition}

\begin{proposition}\label{proposition:pos}
There exist problem instances where $\PoS = \Omega(\lambda |\cR|)$, where $\lambda$ is the alternative travel cost of a user. 
\end{proposition}
\begin{proof}
Consider an example where there are a driver $d$ and a rider $r$, and $d$ can serve $r$ at zero cost. However, there's another rider $r'$ who lives far away and another driver $d'$ who lives close to $r$ but farther from $r'$ than $d$. Suppose $(d',r')$ is infeasible ($d'$ and $r'$ cannot be matched together) and alternative cost for every user is $Q>>0$.
Then in the min-cost solution, we have $(d,r')$ and $(d',r)$, incurring cost $\epsilon/2$ for each user ($d,d',r,r'$). 
Then in the min-cost solution, there are  $(d,r')$ and $(d',r)$ which incur a total cost of $2\epsilon$, whereas only stable matching gives utility of $2\epsilon$. However, $(d,r)$ forms a blocking pair. Since $(d',r')$ is infeasible, then $d'$ and $r'$ both incur a cost of $Q$. Thus, the total cost in the stable outcome (where only $(d,s)$ is matched) is $2Q$. This leads to $\PoS = 2Q/2\epsilon = Q/\epsilon$, which can be arbitrary large.
\end{proof}


\subsection{Fairness-Stability Trade-off}
By enforcing $\theta$-fairness, the P2P ridesharing system may become unstable, as more costly riders may be matched to drivers so that drivers are more likely to deviate from the given matching. We formalize the effects of imposing $\theta$-fairness on the stability of the matching in Proposition \ref{prop:fairstable}.

\begin{proposition}\label{prop:fairstable}
$\forall$ $\theta >0$, a $\theta$-fair solution is not ex-post stable.
\end{proposition}
\begin{proof}
Consider the following example with 1 driver, $m$ riders and a fairness level $\theta>0$. The driver $d$ has a very flexible time window and a capacity of 2. 
There are $m-2$ riders with strict and narrow window so that serving each one of them incur a cost of $Q$. 
The driver cannot serve 2 of them simultaneously. More formally we have rider $r_i$ with window $[i,i+m]$ and the $\dist_r = m$ for all $i\in [3,m]$. The driver can serve any of them with his flexible window at a cost $Q$, but cannot serve two of them simultaneously due to their strict time windows. We also have 2 more riders $r_1$ and $r_2$ that can be served together, but if the driver serves $r_1$ only, he incurs a total cost of $\epsilon$, whereas serving $r_2$ only or $(r_1,r_2)$ together will incur a cost of $Q$. First, we study the cost-minimizing matching, where the driver only serves the rider $r_1$, incurring a cost of $\epsilon$. However, the cost-minimizing matching is not $\theta$-fair. The $\theta$-fair solution requires the driver to serve all riders with probability $\geq \theta$, $\theta \in (0,\frac{1}{m-1}]$. This matching can be obtained by serving $m-2$ riders with probability $\theta$, serving $(r_1,r_2)$ pair with probability $\theta$ and serving $r_1$ with probability $1-(m-1)\theta$ at cost $\epsilon$. This results in a total cost of $\theta Q+\epsilon -(m-1)\theta\epsilon$, and therefore incurs $\PoF \geq \frac{\theta Q}{\epsilon}$ (which can be arbitrarily large). 
When the driver's altruistic factor $\rho_d$ is finite, the driver $d$ and $r_1$ may form a blocking pair in which each of them gets a higher utility and incurs a lower cost. Thus, the $\theta$-fair solution is not \textit{ex-post} stable. 
\end{proof}

\section{Experiments}

\begin{figure}[ht!]
    \centering
        \begin{minipage}{0.5\textwidth}
        \centering
        \setlength{\belowcaptionskip}{-4pt}
\begin{tikzpicture}[scale=3.1]

\draw[color = purple!10!white, fill=purple!10!white] (0.1,0) rectangle (2.85,0.75);
\draw[color = blue!20!white, fill=blue!20!white] (0.1,0) rectangle (0.89,0.29); 
\draw[color =green!20!white,fill=green!20!white] (2.1,0.45) rectangle (2.85,0.75); 
\draw[color=red!20!white,fill=red!20!white] (1.2,0) rectangle (2.1,0.22); 
\draw[color=orange!20!white,fill=orange!30!white] (0.88,0.31) rectangle (1.65,0.55); 
\draw[color=purple!30!white,fill=purple!30!white] (0.1,0.35) rectangle (0.85,0.75); 
\draw[color=gray, fill=gray] (1.35,0.35) rectangle (1.69,0.53);  


\node[align=center] at (1.515,0.45) {{\Huge$\star$}};
\node[align=center, color = blue] at (0.40,0.15) {A \\7:12};
\node[align=center, color = blue] at (0.62,0.15) {\footnotesize 40\%};
\node[align=center, color = green!40!black] at (2.4,0.6) {B\\ 7:20};
\node[align=center, color = green!40!black] at (2.67,0.6) {\footnotesize 20\%};
\node[align=center, color = red] at (1.55,0.11) {C\\ 7:40};
\node[align=center, color = red] at (1.77,0.11) {\footnotesize 10\%};
\node[align=center, color = orange] at (1.0,0.43) {D\\ 8:00};
\node[align=center, color = orange] at (1.22,0.43) {\footnotesize 20\%};
\node[align=center, color = purple] at (0.5,0.55) {E\\ 8:00};
\node[align=center, color = purple] at (0.72,0.55) {\footnotesize 10\%};

\end{tikzpicture}
\caption{A graphical description of the experimental setting.
}
\label{fig:expSetting}
    \end{minipage}
\end{figure}
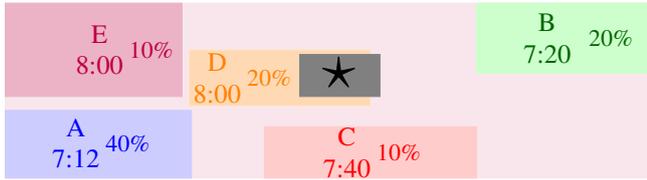

\begin{figure*}[t]
\centering
    
\subfigure[Driver/Rider Ratio]{\label{subfig:ratio}
    \includegraphics[width=0.26\textwidth,height=85pt]{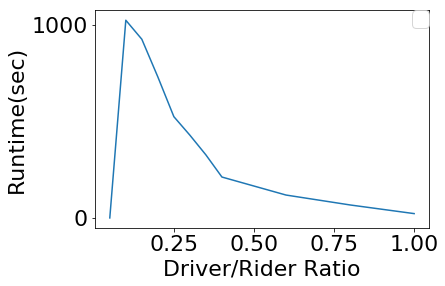}
    }
\subfigure[IR and Stability]{\label{subfig:scalIR}
    \includegraphics[width=0.28\textwidth,height = 85pt]{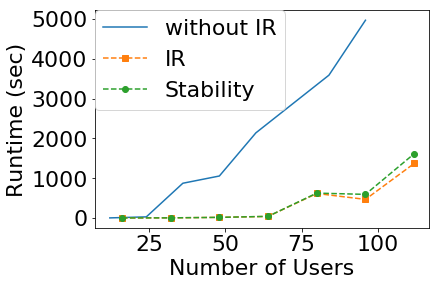}
    }
\subfigure[Fairness Runtime]{\label{subfig:scalFair}
    \includegraphics[width=0.24\textwidth,height = 85pt]{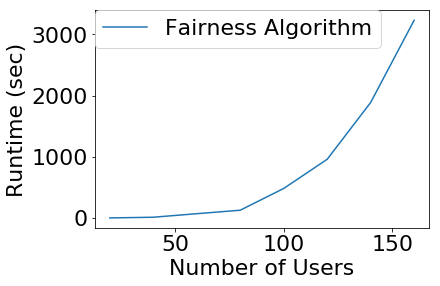}
    }
    \caption{Results on synthetic data of a morning rush hour.
    }
	\label{fig:scal}
\end{figure*}
\begin{figure*}[ht!]
\centering
    
\subfigure[Cost of Heuristics]{\label{subfig:Hbar}
    \includegraphics[width=0.34\textwidth,height=75pt]{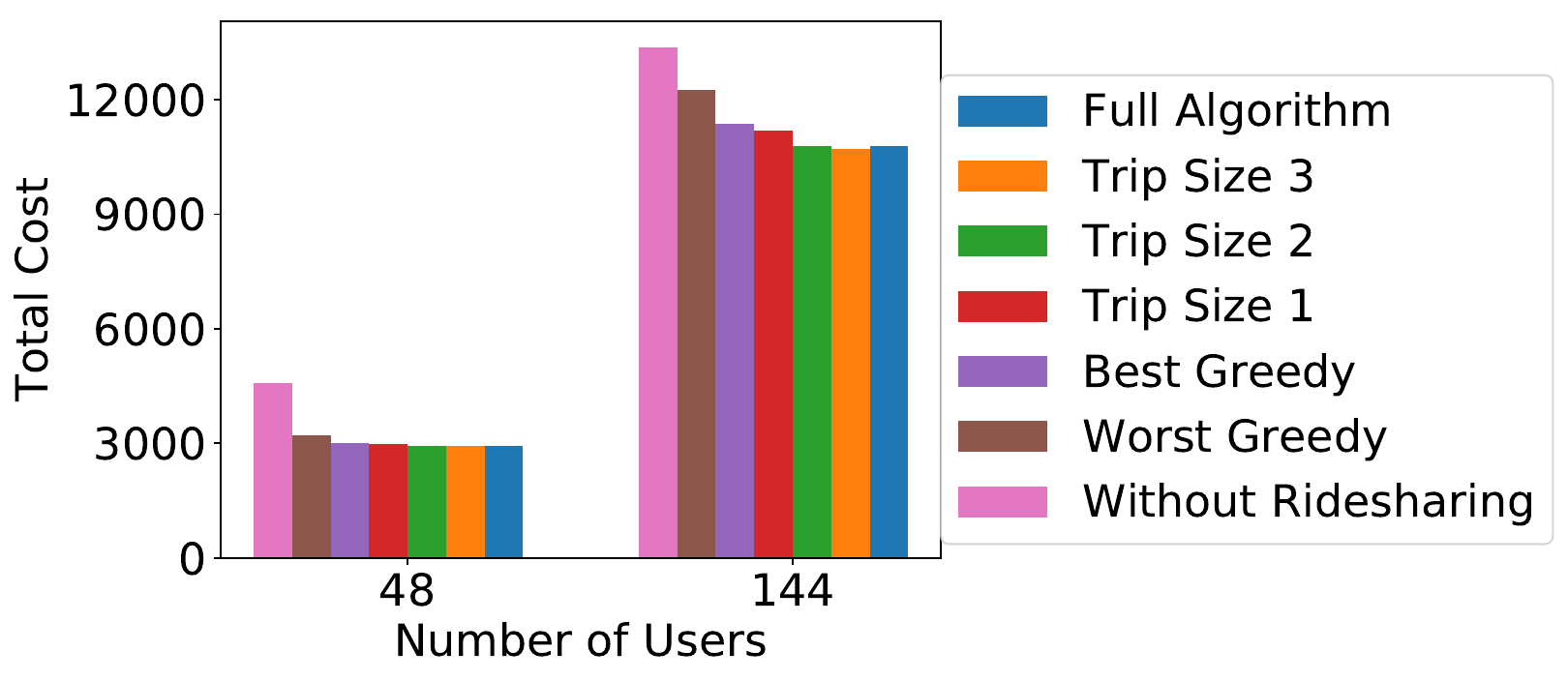}
    }
\subfigure[Runtime Heuristics]{\label{subfig:Htline}
    \includegraphics[width=0.25\textwidth,height = 75pt]{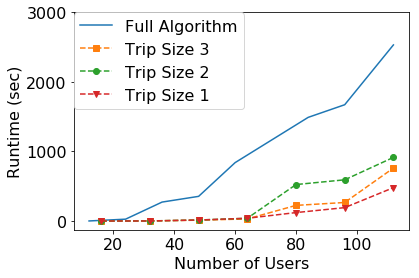}
    }
\subfigure[Large-size Scalability]{\label{subfig:scalsparse}
    \includegraphics[width=0.25\textwidth,height=75pt]{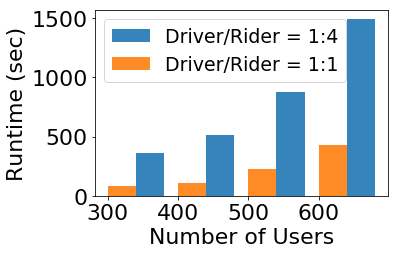}
    }
    \caption{Results based on synthetic data of a morning rush hour (a)-(b) and a large-scale setting (c).
    }
	\label{fig:scal_large}
\end{figure*}

\begin{figure*}[ht!]
\centering

\subfigure[Reduced Travel Time]{\label{subfig:RTT}
    \includegraphics[height=85pt]{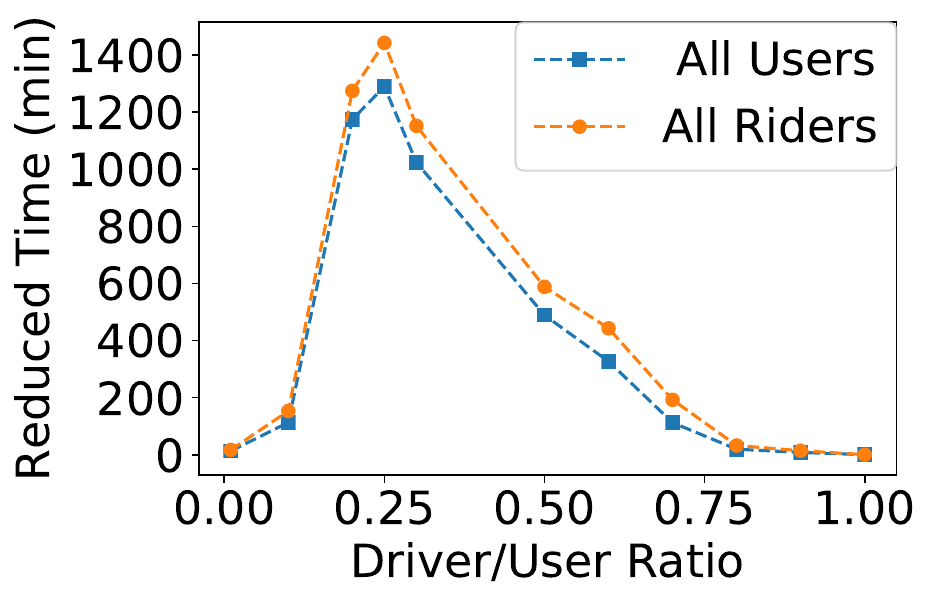}
    }
\subfigure[PoS and PoF]{\label{subfig:pospof}
  \includegraphics[height=85pt]{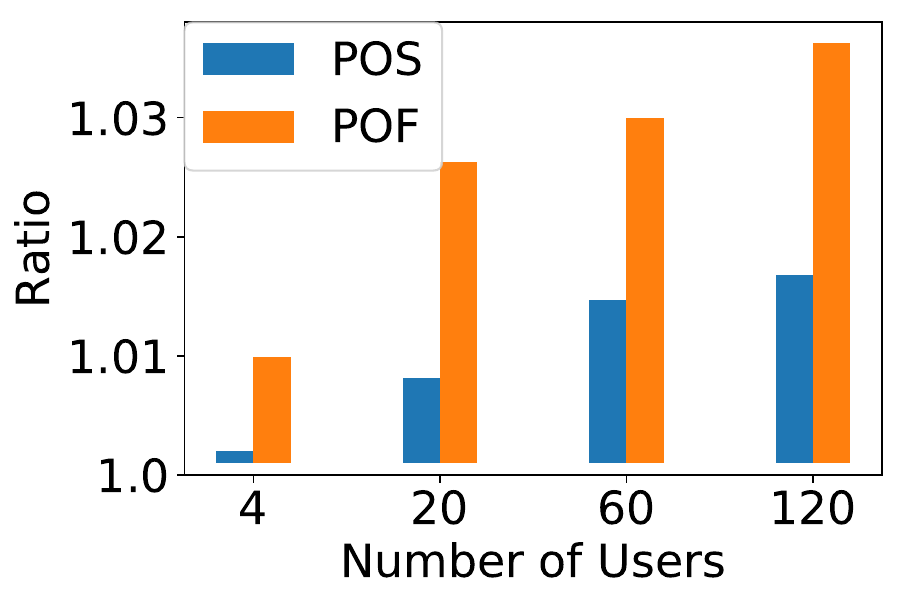}
    }
\subfigure[Effect of Extra Utility]{\label{subfig:altruism}
    \includegraphics[height=95pt]{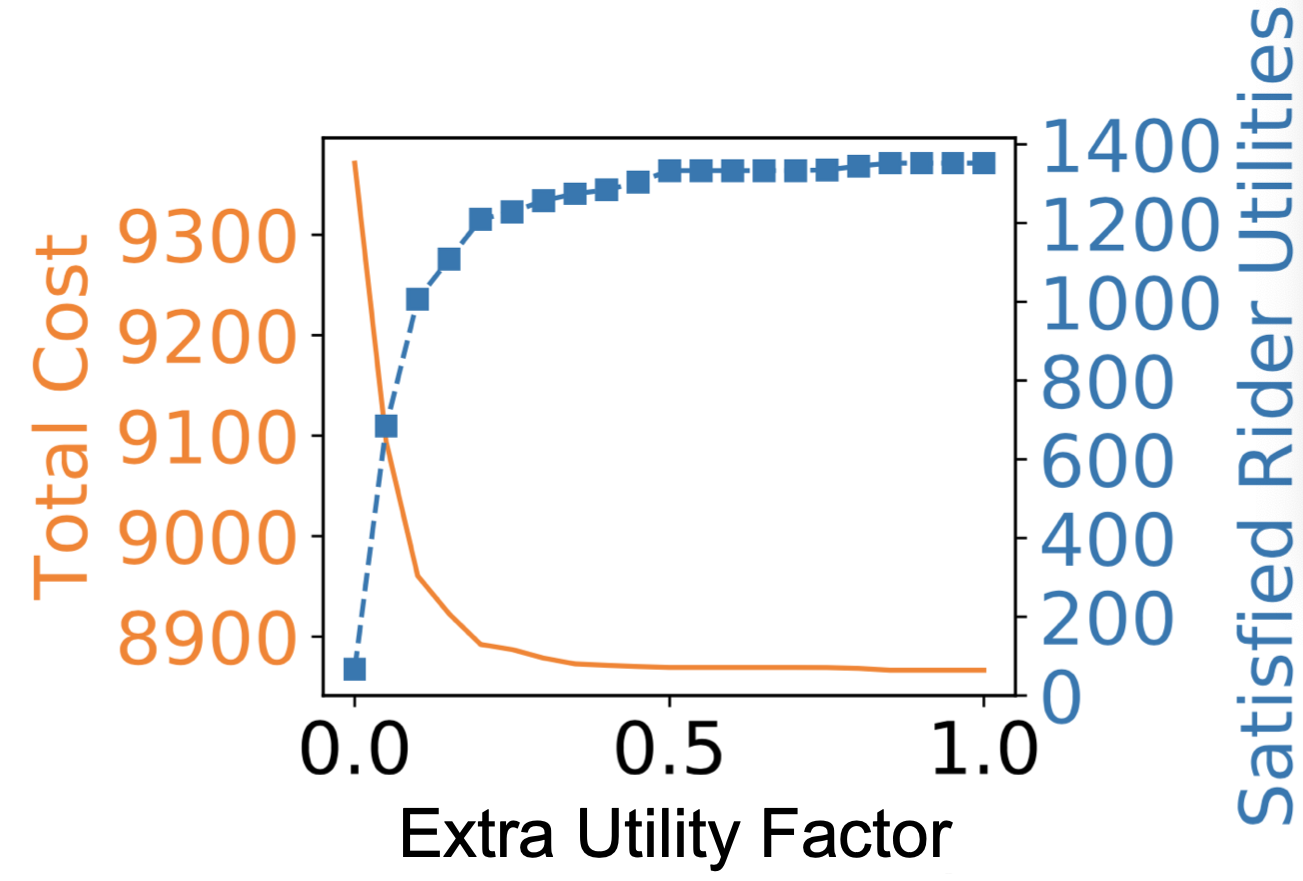}
	}
	\caption{Results on real-world data based setting}
	\label{fig:algFig}
\end{figure*}

We evaluate our algorithms with three sets of experiments. The first one is a base setting, the second one is stress-testing the scalability of the algorithms, and the third set is based on real-world data.
All BIPs and LPs are solved by Gurobi Optimizer on a 3.1 GHz Xeon E5 desktop with 16 GB RAM. All results are averaged over 20 instances. 

First, we experimented with problem instances that simulate a typical neighborhood in a morning rush hour. We used a $50\times50$ graph $G$. We define $\dist$ as the Euclidean distance. Users are randomly generated and drawn from 5 different user types. The base setting and details about different types of users are described in Figure \ref{fig:expSetting}. The gray $\star$ region correspond to a neighborhood, i.e. the origin of all users. The other regions are the destinations of different types of users, with their percentages and latest departure times marked on the graph. For example, 40\% of all users are of type A. A type-A user $i$'s destination is drawn uniformly from the blue region in the bottom-left corner, $\tau^e_i$ drawn from [7:00,7:12], $\tau^\star_i$ drawn from $[\tau^e_i,$7:12], $\tau^l_i =$ 7:12+$\dist_i$. Where 7:00 is the earliest possible departure time of all users, and $\max_{u,v\in V}\dist(u,v)$ is 60 minutes.
Travel costs ($c^i_{\text{trl}}$) are set to 3 per minute, and deviation costs ($c^i_{\text{dev}}$) are 1 per minute. All drivers have a capacity of $k_d= 4$, and altruistic factor $\rho_d = 1.2$. The value of the user $i$ is set to $\val_i = c^i_{trl}\cdot \dist_i \cdot U[1,2.5]$.
Lastly, the alternative cost of the trip, i.e. $\lambda_i$, is set to $\val_i$. 
Figure \ref{subfig:ratio} shows how runtimes change as the driver-to-rider ratio varies (with 50 users). 
Figure \ref{subfig:scalIR} shows that adding IR and stability constraints reduces the runtimes significantly. 
Indeed, fairness and stability reduce the number of feasible matchings; therefore, the algorithm can scale up to a significantly larger number of users (110 users within 30 minutes). 
This suggests that IR and stability enhance both the practical benefits of the solution and the scalability of the algorithm.
Figure \ref{subfig:scalFair} shows the runtime for our fairness algorithm. Even though the fairness algorithm uses our baseline algorithm, the runtime did not increase significantly. 

Next, we focus on using a heuristic that restricts the maximum Trip Size to be 1, 2, or 3. Figure \ref{subfig:Hbar} compares solution quality for different methods. The greedy algorithm, as a benchmark, involves randomly ordering drivers and greedily matching them to riders. We ran $|\cD|^2$ orderings and took the best and worst solutions to compare with the full algorithm. 
These results underscore the benefits of our algorithm which can result in more efficient matching and lower overall costs to the system. While Figure~\ref{subfig:Hbar} shows that these heuristics are slightly worse than the full algorithm, Figure \ref{subfig:Htline} shows that they accelerate runtime, as compared to the full algorithm.
For instance, the Trip Size 2 has increased total costs by only 2\%. 
This highlights opportunities to derive near-optimal solutions in much faster computational times in practice.

In addition to our small scale baseline experimental setting (which mimics a typical morning rush hour situation), we also test our algorithm in a large-sized setting(which captures suburbs and rural areas). In this experiment, the origins $o_i$ and destinations $d_i$ of user $i$ are randomly drawn from the uniform distribution within the graph $G$. User earliest departure times $t^e_i$, latest arrival times $t^l_i$, preferred times $t^\star_i$ are randomly drawn from uniform distribution within their feasible time horizons.
The other parameters remain the same as described in the base setting. Figure \ref{subfig:scalsparse} shows scalability results in this setting. The observations are twofold. First, the runtimes scale linearly; second, the algorithm is much more scalable in sparse large-sized settings. 

Moreover, we also experimented with real-world data collected from a survey to residents living in subsidized, low-rent apartments in a suburban area in the US. Their responses not only indicate their interest in participating P2P ridesharing but also describe their regular destinations, preferred travel time, etc. We extracted 8 locations
that are frequently visited on a typical weekend and extracted $\lambda_i$ as travel time using public transportation time and driving time respectively from Google Maps. The community has 78 apartments, and we expect 125 residents to participate in the P2P ridesharing program. Additional parameter settings can be found in the appendix.
Fig. \ref{subfig:RTT} shows that our algorithm can efficiently match drivers and riders, reducing total travel time by more than 20\% (1400+ minutes).
Figure \ref{subfig:pospof} shows the price of stability ($\PoS$) and the price of fairness ($\PoF(\theta=0.2)$) with different numbers of users and a driver to rider ratio of 1:4. Even though theoretically $\PoS$ and $\PoF(\theta)$ can be arbitrarily large, in practice incorporating fairness and stability considerations results in added costs of only less than 4\% and 2\%, respectively. 
Figure \ref{subfig:altruism} shows the effects of the extra utility factor $(\rho)$ on total cost and satisfied rider utilities. 
Even a small extra utility factor of 0.5 can achieve a reduction in cost at a similar level of higher values. 
All these results show that our models and algorithms can derive high-quality solutions in large-scale instances, and elicit the trade-offs between critical performance measures to support decision-making in the P2P ridesharing systems.
\hspace{-6pt}
\section{Discussion and Conclusion}

Some of the assumptions in our model can be further relaxed. For example, we may add flexibilities in user roles with minor amendments. With role flexibility, a user may indicate his interest in participating in the P2P ridesharing system and is willing to drive others and/or get a ride from others. Since a user may be either a driver or a rider, we duplicate the user to get a driver and a ``mirroring'' rider. 
In the matching BIP, we can add a constraint $y_{r_d}=\sum_{S\in\cS}x_{dS}$ for a driver $d$ and his mirroring rider $r_d$ to ensure that the flexible user is either driving or acting as a rider. It should be noted that the rider set expanded to $\cR^{f}=\cR\cup\{r_d\}_{d\in\cD}$ and $\cS\subseteq 2^{\cR^f}$.

We can also easily extend our model so that a driver's extra utility may not be proportional to the utility of the matched riders --- only slight modifications to our algorithms are needed to handle the case (see appendix for more details).


To conclude, we proposed a novel user-centric approach to the matching problem in non-commercial P2P ridesharing, taking into account efficiency, fairness, and stability. 
The models and algorithms developed in this paper have now been implemented in a real-world P2P ridesharing system.




    

\newpage
\bibliography{uai2021-template}

\newpage
 \appendix
\input{appendix_uai2021}
\end{document}

%% file: appendix_uai2021.tex
\section*{Appendix}
\section{Notation Table}
\begin{minipage}{0.9\linewidth}
\centering
\begin{tabular}{|c|c|}
\hline
$[0,T]$ & Time horizon\\
\hline
$\cR$/$\cD$ & Set of riders/drivers\\
\hline
$r$/$d$ & One rider/driver\\
\hline
$o_{r/d}$, $q_{r/d}$& Origin, Destination of a rider/driver\\
\hline
$V$& $\{o_i \cup q_i:i\in \cD\cup\cR\}$\\
\hline
$W_{r/d}$ & Time window of a rider/driver\\
\hline
$\tau^e_{r/d}$ & Earliest departure time for a rider/driver\\
\hline
$\tau^l_{r/d}$ & Latest arrival time for a rider/driver\\
\hline
$\tau^\star_{r/d}$ & Preferred departure time for a rider/driver\\
\hline
$\Delta_{r/d}$ & Maximum detour time for a rider/driver\\
\hline
$\val_{r/d}$ & Value gain by a complete the trip\\
\hline
$k_d$ & Capacity of the driver $d$\\
\hline
$\rho_d$ & The altruistic factor of the driver $d$\\
\hline
$\cS = 2^{\cR}$ & Set of all possible subsets of riders\\
\hline
$c_{dS}$ & Cost of driver $d$ serving $S\in\cS$ \\ 
\hline
$\lambda_r$ & Cost of alternative transport mode for rider $r$\\
\hline
$\dist(u,v)$ & Distance from node $u$ to node $v$\\
\hline
$\dist_r$ & $\dist(o_r,q_r)$; Distance from $o_r$ to $q_r$\\
\hline
$C^d_{tt'}$ & Total cost for the driver $d$ when departures\\ & at time $t$ and arrives destination at time $t'$\\
\hline
$C^r_{tt'}$ & Total cost for the rider $r$ with picked-up\\ & time $t$ and dropped-off time $t'$\\
\hline
$c^{r/d}_{dev}$, $c^{r/d}_{trl}$ & Deviation/travel cost per unit time\\
\hline
$\Pi$/$\cM$ & The set of all schedules/matchings\\
\hline
$C^{i,\pi}$ & Total cost for user $i$ given schedule $\pi$\\
\hline
$\pi$ & A system schedule (ordered (stop,time) pairs)\\
\hline
$M$ & Matching; assignment of $\cD$ to $\cS$\\
\hline
$\route$ & A driver route (ordered stops for a driver)\\
\hline
$S^{\pi/M}_d$ & Set of riders that driver $d$ is matched in $\pi$/$M$\\
\hline
$d^{\pi/M}_r$ & The driver that rider $r$ is matched in $\pi$/$M$\\
\hline
$S^{\pi/M}_r$ & The set rider $r$ is matched in $\pi$/$M$\\
\hline
$N(\route)$ & Possible next stop of the route\\
\hline
$\dschedule_h$ & Heuristic driver Schedule\\
\hline
$c_h$ & Cost of a heuristic driver schedule\\
\hline
$\route_{end}$ & The last node in the route\\
\hline
$\route + u$ & Add node $u$ at the end of $\route$\\
\hline
$r(v)$ & The rider that the node $v$ corresponds to \\
\hline
$\nxt^{\route}_{v}$ & The node that's next to $v$ in $\route$.\\
\hline
$\CurPas_{\route}(v)$  & The set of riders in the car at node $v \in \route$\\
\hline
\end{tabular}  
\end{minipage}

\section{Efficient Matching}

\subsection{NP-hardness}

The P2P ride-sharing problem is NP hard, even when we don't consider the deviation cost, the capacity limits, and the IR constraints, etc. The P2P ridesharing problem can be reduced from the TSP problem. Suppose that each node $v$ in the TSP problem is a rider $r$'s origin and his destination as well, while his time window is $[0,\infty]$ and $\val_r = \infty$. Furthermore, suppose there is only one driver, and the  driver's origin and destination are also a same vertex $v$. Moreover, suppose his altruistic factor is $\rho = \infty$ and his time window is $[0,\infty]$. In this setting, the driver's utility with altruism is maximized when he tries to serve all riders. Thus, the problem setup is equivalent to traverse all the nodes in the graph, subject to minimizing distance travelled. Note that this reduced setup is exactly a TSP problem.


\subsection{Details about the RTV Framework}
To compute all feasible $(d,S)$ pairs for all drivers, we need to construct the rider-trip-vehicle-graph (RTV-graph)~\cite{alonso2017demand}. The construction of the RTV graph requires the construction of rider-vehicle-graph(RV-graph), decomposition, and finally the construction of RTV-graph.

\subsubsection{RV-graph}
In the first step, we compute the RV-graph as follows: For all pairs of $(d,r)$, we check whether the driver $d$ can pick up the rider $r$ and drop off him within both of their time windows while satisfying their IR constraints. If feasible, we add an edge between $d$ and $r$. Furthermore, for each rider pair $(r_1,r_2)$, we check whether any virtual driver can pick up both riders and drop them off within their time windows. If they are feasible we add an edge between $r_1$ and $r_2$.

\subsubsection{Decomposition}
In order to solve large problem instances quickly, we want to solve mutually independent smaller problems in parallel if possible. To construct such mutually independent sub-instances, we start by treating each driver as its own group and then go through all the riders he can possibly serve. If no two drivers can serve one rider, then the two drivers are not in the same group. we merge the two groups $d$ and $d'$ are in. We continue these operations until no further groups can be combined. Then, we run the algorithm for each subgroup of drivers and their compatible riders in parallel.

\subsubsection{RTV-graph}
Next, we construct the RTV-graph. The RTV graph has 3 types of nodes --- the driver nodes $\forall d\in \cD$, the trip nodes for all (feasible) subset of riders $S\subseteq \cR$, and the rider nodes $\forall r\in \cR$. At the beginning, we add an empty trip. Next, we add edges between $d$ and empty trip for all $d\in \cD$.
Then we add trip $S$ of size 1. We have edge between $r$ and $S$ if $r\in S$. And we have an edge between $S$ and $d$ if $(d, S)$ is feasible. 
Let $\cS_1$ be set of the feasible trip of size 1. Similarly, we will maintain a feasible trip of size $i$ as $\cS_i$ for $i\geq 2$. 
For trip size $i \in \{2, ... ,|\cR|\}$, for each trip $S' \in \cS_{i-1}$, add a rider $r$ from feasible rider $\cS_1$. Let $S$ = $S' \cup \{r\}$. If $|S| < i$ and any subset of $S$ with size $i-1$ is not in $\cS_{i-1}$ then continue to next rider in $\cS_1$. Otherwise compute $\TripCost(d,S)$

\subsection{Proof of Observation \ref{obs1:equivalence}: Cost Minimization}

Let $S^\pi_{\cD}:=\cup_{i\in\cD} S^\pi_i$ be the set of riders that are matched to any driver in $\cD$ in the schedule $\pi$. First, we consider the case where drivers are not altruistic.

\begin{align*}
&&\max_{\pi \in \Pi} \left( \sum_{r\in \cR} U_r + \sum_{d\in \cD} U_d \right)
\quad = \quad \sum_{r\in \cR} \val_r + \sum_{d\in \cD}\val_d \\
&  &- \sum_{r\in S^\pi_{\cD}} C^r_{\overline{t^\pi_{r}}\underline{t^\pi_{r}}} 
- \sum_{r\in\cR:r\not\in  S^\pi_{\cD} }\lambda_r
-\sum_{d\in \cD} C^d_{\overline{t^\pi_{d}}\underline{t^\pi_{d}}} 
\end{align*}

Note that $\sum_{r\in \cR} \val_r$ and $\sum_{d\in \cD}\val_d$ are independent of $\pi$. Therefore the above optimization is equivalent to the following cost minimization problem.
\begin{align*}
\min_{\pi \in \Pi} \sum_{d\in \cD} \sum_{r\in S^\pi_d} C^r_{\overline{t^\pi_{r}}\underline{t^\pi_{r}}}  +\sum_{d\in \cD} C^d_{\overline{t^\pi_{d}}\underline{t^\pi_{d}}} + \sum_{j\in \cR: j\not\in S^\pi_{\cD}} \lambda_j
\end{align*}

Now consider the altruistic case, where a driver's utility contains a term which is linearly proportional to the utility sum of the rider subset he serves.  

\begin{align*}
&\max_{\pi \in \Pi} \left( \sum_{r\in \cR} U_r + \sum_{d\in \cD} \tilde{U}_d \right)
\quad = \quad \sum_{r\in \cR} \val_r + \sum_{d\in \cD}\val_d \\&- \sum_{r\in\cR:r\not\in  S^\pi_{\cD} }\lambda_r + \sum_{d\in \cD}  \sum_{r\in S^\pi_d} \rho_d (C^r_{\overline{t^\pi_{r}}\underline{t^\pi_{r}}}-\val_r)
\\&- \sum_{r\in S^\pi_{\cD}} C^r_{\overline{t^\pi_{r}}\underline{t^\pi_{r}}} 
-\sum_{d\in \cD} C^d_{\overline{t^\pi_{d}}\underline{t^\pi_{d}}} 
\end{align*}

Again, $\sum_{r\in \cR} \val_r$ and $\sum_{d\in \cD}\val_d$ are independent of $\pi$. Thus, the social welfare maximization problem with driver altruism can be described as the following cost minimization problem.
\begin{align*}
&\min_{\pi \in \Pi}& \sum_{d\in \cD}  \sum_{r\in S^\pi_d} (1+\rho_d) C^r_{\overline{t^\pi_{r}}\underline{t^\pi_{r}}}  + \sum_{d\in \cD} C^d_{\overline{t^\pi_{d}}\underline{t^\pi_{d}}} \\ &&+ \sum_{j\in \cR: j\not\in S^\pi_{\cD}} \lambda_j - \sum_{d\in\cD}\sum_{r\in S^\pi_d}\rho_d\val_r \
\end{align*}

Alternately we can have external incentive instead of altruistic factors. If external incentive is proportional to utility of satisfied riders, then we get the same equation as above. Thus we only look at external incentive factor that is proportional to number of satisfied riders.

\begin{align*}
&\max_{\pi \in \Pi} \left( \sum_{r\in \cR} U_r + \sum_{d\in \cD} \tilde{U}_d \right)
\quad = \quad \sum_{r\in \cR} \val_r + \sum_{d\in \cD}\val_d \\&- \sum_{r\in\cR:r\not\in  S^\pi_{\cD} }\lambda_r + \sum_{d\in \cD}  \sum_{r\in S^\pi_d} 1 
- \sum_{r\in S^\pi_{\cD}} C^r_{\overline{t^\pi_{r}}\underline{t^\pi_{r}}} 
-\sum_{d\in \cD} C^d_{\overline{t^\pi_{d}}\underline{t^\pi_{d}}} \\
& = \sum_{v\in \cR} \val_r + \sum_{d\in \cD}\val_d -\sum_{r\in \cR:r\not\in S^\pi_{\cD}} (\lambda_r+1) + |\cR| 
- \sum_{r\in S^\pi_{\cD}} C^r_{\overline{t^\pi_{r}}\underline{t^\pi_{r}}} 
\end{align*}

Note that $|\cR|$, $\sum_{r\in\cR}\val_r$ and $\sum_{d\in\cD} \val_d$ are independent of $\pi$. Thus, the maximization problem is equivalent to following cost minimization problem with modified unsatisfied cost coefficient.

\begin{align*}
&\min_{\pi \in \Pi}& \sum_{d\in \cD}  \sum_{r\in S^\pi_d}  C^r_{\overline{t^\pi_{r}}\underline{t^\pi_{r}}}  + \sum_{d\in \cD} C^d_{\overline{t^\pi_{d}}\underline{t^\pi_{d}}} + \sum_{j\in \cR: j\not\in S^\pi_{\cD}} (\lambda_j+1)
\end{align*}

\subsection{Pruning}

Pruning is essential to our algorithm. The first pruning condition (Line \ref{treesearch:pruning1}) is when the lower bound of the cost of a partial route $\underline{c}$ is no better than the current best solution $(c^*, z^*)$, which is initialized with $(c^h,z^h)$. There are various ways to get a lower bound and we choose to ignore the deviation cost and only compute the traveling time cost, which can be computed in an incremental way (Line \ref{treesearch_lowerbound}). Including the additional travel cost for visiting some of the remaining stops can lead to a tighter bound.
In addition, we prune the branches using the constraints of vehicle capacity (implicitly in Line \ref{treesearch:nextstop}). Further, by recording the earliest possible time for visiting the last stop in the partial route $\underline{\tau}$, we can prune the branch using the latest arrival time constraint of all the current users on the vehicle (Line \ref{treesearch:pruning2}). Lastly, we prune a branch if the maximum detour time constraint for any rider on the vehicle will be violated (Line \ref{treesearch:pruning3}). The function $\delta(i,\route)$ computes the remaining detour budget for rider $i$, i.e., the maximum additional detour time rider $i\in S$ can afford given the current partial route $\route$.
With these pruning techniques, $\TripCost$ is able to compute $c_{dS}$ and the corresponding $z_{dS}$ efficiently. Specifically, $\TripCost$ is guaranteed to find the optimal solution --- although the worst-case runtime is exponential. In practice, however, it runs efficiently for middle-sized problems due to the effective pruning, as we show in our experiments. 

\begin{algorithm}[t]
  \caption{$\mathsf{TreeSearch}
(d,S,c^*,\dschedule^*,\route,\underline{\tau},\underline{c})$}\label{alg:tripcost_treesearch}
  \begin{algorithmic}[1]
    \If {$|\route| = 2|S|+2$} \Comment{Reach leaf node} \label{treesearch:reach_leaf}
      \If {$\underline{c}<c^*$} 
      \State $(c,\dschedule) \gets \mathsf{Solve}\TripLP(\route)$\label{treesearch:call_TripLP}
      \EndIf
      \If {$c<c^*$}
          \State $(c^*,\dschedule^*)\gets (c,\dschedule)$
      \EndIf
      \State \Return $(c^*,\dschedule^*)$
    \EndIf
    \State $N(\route)\gets$ set of feasible next stops \label{treesearch:nextstop}
    \For {$v'\in N(\route)$}\Comment{Expand the tree}
      \State $\route' \gets \mathsf{append}(\route,v')$\label{treesearch:append}
      \State $\underline{\tau}' \gets \underline{\tau}+\dist(\mathsf{laststop}(\route),v')$\label{alg1:time_nxtstp}
      \State $\underline{c}' \gets \underline{c}+\sum_{i\in \route} c^i_{\text{trl}}\cdot  \dist(\mathsf{laststop}(\route),v')$ \label{treesearch_lowerbound}
      \If {$\underline{c}'<c^*$} \label{treesearch:pruning1} \Comment{Pruning conditions}
      \If {$\underline{\tau}'\leq \min\limits_{i: o_i\in \route \& q_i\notin \route} \tau^l_i$} \label{treesearch:pruning2}
      \If {$\min\limits_{i:o_i\in\route \& q_i\notin \route}\delta(i,\route)\geq \dist(\mathsf{laststop}(\route),v')$}\label{treesearch:pruning3}
          \State $(c,\dschedule)\gets \mathsf{TreeSearch}(d,S,c^*,\dschedule^*,\route',\underline{\tau}',\underline{c}')$ \label{alg1:rec_call_tripSearch}
          \If {$c<c^*$}
              \State $(c^*,\dschedule^*)\gets (c,\dschedule)$
          \EndIf
      \EndIf
      \EndIf
      \EndIf
    \EndFor
  \State \Return $(c^*, \dschedule^*)$
  \end{algorithmic}
\end{algorithm}

\section{Fairness}

\subsection{Proof of Proposition 1}
\begin{proof}
The proof of (a) and (b) follows from global sensitivity analysis from \cite{bertsimas1997introduction}. 

(a) and (b) are direct conclusions from the paragraph below equation (5.2) on page 214. Equation (5.2) states that $F(b)=\max\limits_{i=1,\dots,N}(p^i)^Tb$, $b\in S$, where $F$ is the objective value of the linear optimization problem as a function of right hand side (RHS) vector $b$, $p^1,\dots,p^N$ are the extreme points of the dual feasible set. ``... In particular, $F$ is equal to the maximum of a finite collection of functions. It is therefore a piecewise linear convex function,$\dots$...''. Proof of (a) can also be found in the first paragraph on page 213. The idea is that the optimal basis only changes finite times before the linear program becomes infeasible.

For the proof of (b), also see Theorem 5.1 and its proof on page 213. Let $\theta_a,\theta_b\in[0,1]$ and $\theta_a<\theta_b$. For $i=a,b$, let $x^i$ be an optimal solution to the linear optimization problem. Take a scalar $\lambda\in [0,1]$, $y=\lambda x_a + (1-\lambda) x_b$ is a feasible solution to the linear optimization problem with RHS set to $\bar{\theta} = \lambda \theta_a + (1-\lambda) \theta_b$. Note that the objective function is linear, thus $F(\bar{\theta})\leq \lambda F(x_a) + (1-\lambda) F(x_b)$.

For (c), we develop a bisection search algorithm that finds the exact Pareto Frontier. The Pareto Frontier captures the minimal costs associated with each fairness level $\theta\in [0,1]$; thus, we are looking at a two-dimensional plane where the Pareto Frontier is the envelope of the set of all possible cost-$\theta$ points on the plane.
This method applies to a broader class of problems---as long as the Pareto Frontier is known to be a piece-wise function of finite sub-functions and has non-decreasing second-order gradients. The bisection search algorithm first compute the envelope hyper-planes of the Pareto frontier at boundary points $\theta^0_0 = 0$, and $\theta^0_1 = 1$. 
Without loss of generality, we assume that these envelope planes (lines) intersect, we record the $\theta$-value ($\theta^1_0$) and cost of the intersection of the two lines. We compute the optimal cost at $\theta^1_0$, and compare the optimal cost with the cost of the intersection. From (b) we know that the first-order difference of the Pareto frontier is non-decreasing, thus the true optimal cost at $\theta^1_0$ could either equal or greater than the cost value of the intersection point. If equal, then the upper border of the two intersecting lines is the Pareto frontier on the interval; otherwise, we bisect the interval, and repeat the procedure described above at both of the two half-length closed intervals. Note that there are only finite number of line intersections and base changes in the linear optimization program, this algorithm will stop within finite number of bisection search. (Note that if the two lines coincide at any step, then the line segment between the points where the hyper-planes are taken is the Pareto frontier on the interval.)
\end{proof}
\subsection{Fairness Extension: Different Fairness Levels}

Our fairness model/algorithm can be directly generalized to cases with heterogeneous riders. Suppose in the P2P platform, the probability of being offered a trip is different for different riders. For example, if a rider has previously been a driver who provides trip to other riders, he may be given higher probability of being matched to others who never drives anyone else. In these cases, we will have different threshold $\theta_r$ for different riders. The following algorithm computes a feasible solution that satisfies all the threshold constraints.
\begin{align*}
    \max_{p,\delta z} \quad & z\\
    \text{s.t.}\quad 
    & \delta_i \leq \sum_i m^\ell_i p^\ell &\forall i\in\cR\\
    & z \leq \delta_i - \theta_i & \forall i\in\cR\\
    & \sum_{\ell} p^\ell = 1\\
    & p^\ell \geq 0
\end{align*}

\section{IR and Stability}

\subsection{Riders are automatically IR}
\begin{claim}
Riders are individually rational in the cost-minimization solution.
\end{claim}\label{claim:IR}
\begin{proof}
Suppose for contradiction that a rider $r$ is getting utility lower than $\val_r-\lambda_r$ under the min-cost schedule $\pi^*$. Then consider a schedule $\pi'$ where $r$ is not matched and everyone else is matched the same as in $\pi^*$. In other words, we are picking up and dropping off all riders at the same time, except that we are not picking up the rider $r$ (This schedule may contain unnecessary waiting). It is easy to see that $\pi'$ is a feasible schedule. Because we assume that the stops in a trip scheduled, and times to visit each stop all stay the same, thus the time window constraints, maximum detour times are all satisfied for all the drivers and riders matched in the new schedule $\pi'$. Drivers are assigned to pick up the same set of riders or a subsect of riders he matched in $\pi^*$, thus the capacity constraints in $\pi'$ are satisfied. For each rider $r'\neq r$, he is experiencing exactly the same time deviation and geographic distance traveled in the trip, so his utility stays the same. Moreover, the utility for any driver who was not assigned to pick up $r$ stays the same, because his schedule and the riders that he need to serve stays the same. These holds for both the welfare maximization problem and the efficiency maximization problem.

We only need to consider the utility for the driver $d$ who picked up $r$ in $\pi^*$ (but will spare $r$ in $\pi'$). $d$'s utility is $U_d(\pi^*) = \val_d-C^d_{tt'}$ (or $\tilde{U}_d(\pi^*) = \val_d-C^d_{tt'} + \sum_{r\in S_d} \rho_d (v_r-\lambda_r-C^r_{tt'})$). Recall that we are considering the same schedule as $\pi^*$, i.e. the driver is going to exactly same stops at exactly same times. Therefore, he suffers exactly same deviation cost and travel cost, thus we have $U_d(\pi') = U_d(\pi^*)$, i.e. the driver is getting the same base utility. Everyone except $r$ is getting same utilities and $r$ is getting better utility in $\pi'$. (And therefore, $\tilde{U}_d (\pi')\geq \tilde{U}_d(\pi^*)$.) Thus $\pi'$ is strictly superior to $\pi^*$, this is a contradiction because $\pi^*$ is assumed to be the cost-minimization (as well as welfare maximization) schedule.
\end{proof}

\subsection{Proof of Proposition 5: Non-stable Systems}
\usetikzlibrary{arrows,automata}
\begin{center}
    
\begin{tikzpicture}[scale=0.7]
\node[state] at (0,7) (oa) {$o_a$};
\node[state] at (0,5) (ob) {$o_b$};
\node[state] at (0,3) (oc) {$o_c$};

\node[state] at (7,7) (qa) {$q_a$};
\node[state] at (7,5) (qb) {$q_b$};
\node[state] at (7,3) (qc) {$q_c$};

\node[state] at (10,5) (qq) {$q$};

\draw (oa) edge node[above]{3} (qa);
\draw (ob) edge node[above]{3} (qb);
\draw (oc) edge node[above]{3} (qc);

\draw (oa) edge node[left]{2} (ob);
\draw (ob) edge node[left]{2} (oc);
\draw (oc) edge[bend left=35] node[left]{2} (oa);

\draw (qa) edge node[left]{2} (qb);
\draw (qb) edge node[left]{2} (qc);
\draw (qc) edge[bend right=35] node[pos=0.7,left]{2} (qa);

\draw (qa) edge node[above]{1} (qq);
\draw (qb) edge node[above,pos=0.7]{1} (qq);
\draw (qc) edge node[above]{1} (qq);

\end{tikzpicture}
\end{center}
\begin{center}
Drivers\quad
\resizebox{\columnwidth}{!}{
\begin{tabular}{ | c | c | c | c | c|c|c|c|c|c|c| }
\hline
     \textcolor{black}{} & \textcolor{black}{ori} & \textcolor{black}{des} & \textcolor{black}{window} & \textcolor{black}{$\tau^*$} & \textcolor{black}{cap} & \textcolor{black}{$c_{dev}$} & \textcolor{black}{$c_{trl}$} & \textcolor{black}{value} & \textcolor{black}{$\rho$} & \textcolor{black}{$\Delta$}\\ 
    \hline
    $d_1$ & $o_a$ & $q$ & $[0,10]$ & $0$ & $2$ & $100$ & $1$ & $1000$ & $1$ & $10$\\
    \hline
    $d_2$ & $o_b$ & $q$ & $[0,10]$ & $0$ & $2$ & $100$ & $1$ & $1000$ & $1$ & $10$\\
    \hline
    $d_3$ & $o_c$ & $q$ & $[0,10]$ & $0$ & $2$ & $100$ & $1$ & $1000$ & $1$ & $10$\\
    \hline
\end{tabular}
}
\end{center}

\begin{center}
Riders\quad
\resizebox{\columnwidth}{!}{
\begin{tabular}{ | c | c | c | c | c|c|c|c|c|c| }
\hline
     \textcolor{white}{} & \textcolor{black}{ori} & \textcolor{black}{des} & \textcolor{black}{window} & \textcolor{black}{$\tau^*$} &  \textcolor{black}{$c_{dev}$} & \textcolor{black}{$c_{trl}$} & \textcolor{black}{value} &  \textcolor{black}{$\Delta$} &  \textcolor{black}{$\lambda$}\\ 
    \hline
    $r_1$ & $o_a$ & $q_a$ & $[1,8]$ & $1$ & $5$ & $1$ & $70$ & $7$ & $70$\\
    \hline
    $r_2$ & $o_b$ & $q_b$ & $[1,8]$ & $1$ & $5$ & $1$ & $70$ & $7$ & $70$\\
    \hline
    $r_3$ & $o_c$ & $q_c$ & $[1,8]$ & $1$ & $5$ & $1$ & $70$ & $7$ & $70$\\
    \hline
\end{tabular}
}
\end{center}

Consider the problem with three drivers $\cD = \{d_1,d_2,d_3\}$ and three riders $\cR = \{r_1, r_2, r_3\}$. Their information is described in above tables.
We will look at matching of size 2, 3, and $\geq 4$. 

A matching of size two ($d_i, r_i$): the driver has to deviate one time step to pick up the rider then drives to rider's destination and his destination. Therefore he suffers a deviation cost of 100 then extra traveled cost. More specifically $c^d = 100+3+1 = 104$. Now we will look at the utility of the rider. $U_r = 70-3 = 67$. Then $\tilde{U}_d = 1000-104+67 \leq 1000-4$; therefore, the driver is not IR so the driver will not pick up just one rider. Note picking up $r_{j\neq i}$ gives a worse utility for both the driver and the rider. Thus matching one-to-one is not a feasible matching.

A matching of size three $(d_i, r_i, r_j)$: the driver deviate 1 time step to pick up $r_i$ then $r_j$ then drop off $r_j$ then $r_i$. Thus getting $c^d = 100+2+3+2+1 = 108$. The utility of riders are $U_{r_i} = 70-7 = 63$ and $U_{r_j} = 70-5\cdot 2-3 = 57$. Thus we get $\tilde{U}_d = 1000-108+63+57$ $(>1000-4)$. Therefore, the driver prefer picking up two riders over not being matched. Also note $r_i$ prefers to get matched with $d_i$ over any other drivers.

Furthermore, note $(d_i, r_j, r_k)$ for $j\neq i$ and $k\neq i$ is not feasible due to time window constraint. Similarly, $(d_i, r_i, r_j, r_k)$ is not feasible due to time window and capacity constraint.

Moreover, note if a driver $d_i$ picks up two riders, then no other driver can pick up riders because there's only one rider remaining. 

Therefore, all feasible matching is in from of $(d_i, r_i, r_j)$, $(d_j,\emptyset)$, $(d_k,\emptyset)$, $(null, r_k)$. However, note $(d_j,r_j,r_k)$ form a blocking pair. Therefore, the instance does not have any stable matching. Thus proving there exists an instance where stable outcome is an emptyset.

\section{Details on Experiments}


We have described the setting for the base experiment in the main text. In this section, we describe the settings of the stress-testing scalability experiment and the set of experiments based on real data in detail. 

\subsection{Large-Sized Experiment}

In addition to our small scale baseline experimental setting (which mimics a typical morning rush hour situation), we also test our algorithm in a large-sized setting(which captures suburbs and rural areas). In this experiment, the origins $o_i$ and destinations $d_i$ of user $i$ are randomly drawn from the uniform distribution within the graph $G$. User earliest departure times $t^e_i$, latest arrival times $t^l_i$, preferred times $t^\star_i$ are randomly drawn from uniform distribution within their feasible time horizons, i.e. $[7:00, 8:00 - \dist_i]$. Latest arrival times $t^l_i$ are the earliest departure times $t^e_i$ plus the minimum traveling time $\dist_i$ multiplied by a window flexibility ratio $1.3$, i.e. $t^l_i = t^e_i + \dist_i*1.3$. The user preferred times $t^\star_i = t^e_i + \sigma, \sigma \sim U[0,0.1\dist_i]$ and altruistic factors $\rho_d = 0, \forall d$. And the maximum detour time is set to $\delta_i = 0.2 \dist$. 
 
The other parameters remain the same as described in the base experiment setting. Figure \ref{subfig:scalsparse} shows scalability results in this setting. The observations are twofold. First, the runtimes scale linearly; second, the algorithm is much more scalable in sparse large-sized settings. Indeed, with a driver/rider ratio of 1, our algorithm terminates in 81 seconds for 300 users (150 drivers and 150 riders) and in 428 seconds with 600. In contrast, when the driver/rider ratio is 4, our algorithm terminates in 362 seconds 300 users (60 drivers and 240 riders) and in 21 minutes for 600 users.

\subsection{Real-Data Experiment}

\begin{figure}[h!]
    \centering
    \frame{\includegraphics[width = 0.42\textwidth]{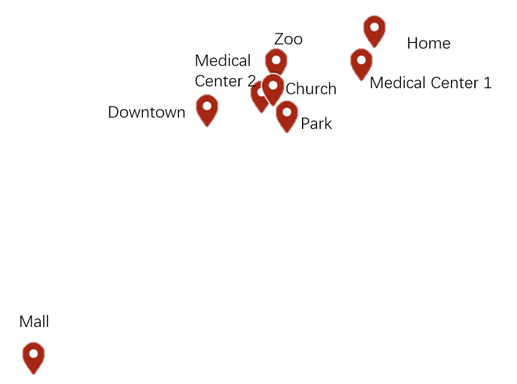}}
    \caption{Locations extracted from a local map used in the real-data experiment}
    \label{fig:HAexample}
\end{figure}

In the third experiment, we tested with problem instances based on information collected from our conversation with residents in a low-income neighborhood where most residents live in subsidized housing. We extracted 8 locations (Home, Church, Medical Center 1\&2, Downtown, Zoo, Park, Shopping Mall) that are frequently visited on a typical weekend. The neighborhood has 78 apartments, and we expect 125 residents to participate in the P2P ridesharing program. 
Based on our conversation with local residents, we learned that around 90\% residents in the neighborhood go to Church on a Sunday morning. Besides the Church trips, the residents may go to Medical Center 1, Medical Center 2, Downtown, Zoo, Park, Shopping Mall with probability 0.1, 0.2, 0.2, 0.1, 0.1, 0.1, respectively. Note that besides the Church trips, with probability 0.2, a resident may not go anywhere. We assume that all driver trips are round-trips, and the length of stay is drawn from $\mathcal{N}(120,60)$. Within the usual travel time window $[7:00, 21:00]$. Based on the length of stay, we draw the earliest departure time $\tau^e_i$ and latest arrival time $\tau^l_i$ uniformly from the feasible time windows. The preferred travel time is set as the earliest departure time, i.e. $\tau^\star_i = \tau^e_i$, for all user $i$. 

In the third experiment, we define $\dist$ as the average travel time in minutes. The travel times between each pair of the 8 locations are proportional to the distances as shown in Figure~\ref{fig:HAexample}, which are average driving time estimated by a well-known commercial statelite map. Travel costs ($c^i_{\text{trl}}$) are set to 1.5 per minute, and deviation costs ($c^i_{dev}$) are set to 0.5 per minute. All drivers have a capacity of $k_d= 4$, and altruistic factor of $\rho_d = 1.2$. The value of the user $i$ is set to $\val_i = c^i_{\text{trl}}\cdot \dist_i\cdot u, u\sim \mathcal{U}[1,2.5]$. Lastly, the unstaisfied cost $\lambda_i$ is set to the travel times of using public transportation.